\documentclass{article}

\PassOptionsToPackage{numbers}{natbib}
\usepackage[final]{nips_2016_mod}

\usepackage[utf8]{inputenc} 
\usepackage[T1]{fontenc}    
\usepackage{hyperref}       
\usepackage{url}            
\usepackage{booktabs}       
\usepackage{amsfonts}       
\usepackage{nicefrac}       
\usepackage{microtype}      

\usepackage{algorithm}
\usepackage{algorithmic}
\usepackage{amsmath}
\usepackage{amssymb}
\usepackage{amsthm}
\usepackage{graphicx}
\usepackage{subcaption}
\usepackage{verbatim}
\usepackage[table]{xcolor}

\newtheorem{lemma}{Lemma}
\newtheorem{theorem}{Theorem}
\newtheorem*{remark}{Remark}

\title{How to scale distributed deep learning?}

\author{
  Peter H.~Jin,
  Qiaochu Yuan,
  Forrest Iandola,
  and Kurt Keutzer \\
  Department of Electrical Engineering and Computer Sciences \\
  University of California, Berkeley \\
  {\texttt{\{phj,qyuan,forresti,keutzer\}@berkeley.edu}} \\
}

\begin{document}

\maketitle

\begin{abstract}
Training time on large datasets for deep neural networks is the principal workflow bottleneck in a number of important applications of deep learning, such as object classification and detection in automatic driver assistance systems (ADAS). To minimize training time, the training of a deep neural network must be scaled beyond a single machine to as many machines as possible by distributing the optimization method used for training. While a number of approaches have been proposed for distributed stochastic gradient descent (SGD), at the current time synchronous approaches to distributed SGD appear to be showing the greatest performance at large scale. Synchronous scaling of SGD suffers from the need to synchronize all processors on each gradient step and is not resilient in the face of failing or lagging processors. In asynchronous approaches using parameter servers, training is slowed by contention to the parameter server. In this paper we compare the convergence of synchronous and asynchronous SGD for training a modern ResNet network architecture on the ImageNet classification problem. We also propose an asynchronous method, gossiping SGD, that aims to retain the positive features of both systems by replacing the all-reduce collective operation of synchronous training with a gossip aggregation algorithm. We find, perhaps counterintuitively, that asynchronous SGD, including both elastic averaging and gossiping, converges faster at fewer nodes (up to about 32 nodes), whereas synchronous SGD scales better to more nodes (up to about 100 nodes).
\end{abstract}

\section{Introduction}






Estimates of the data gathered by a self-driving car start from at least 750 MB/s.%
\footnote{
  \texttt{http://www.kurzweilai.net/googles-self-driving-car-gathers-nearly-1-gbsec}
}
With proper annotation or through an unsupervised learning scheme,
all of this data can become useful for training the object detection system or
grid-occupancy system of a self-driving car.
The resulting training set can lead to weeks or more of training time on
a single CPU/GPU system.
Therefore, for such applications training time defines the most time
consuming element of the workflow, and reduced training time is highly
desirable.

To achieve significant reductions in training time, the training must be
distributed across multiple CPUs/GPUs with the goal of strong scaling:
as more nodes are thrown at the problem, the training time should ideally
decrease proportionally.
There are two primary approaches to distributed stochastic gradient descent
(SGD) for training deep neural networks:
(i) synchronous all-reduce SGD based on a fast all-reduce collective
communication operation
\citep{baidu_deepimage,firecaffe_arxiv,intel_sync,revisiting_sync},
and 
(ii) asynchronous SGD using a parameter server \citep{downpour,msr_adam}.

Both approaches (i) and (ii) have weaknesses at scale.
Synchronous SGD is penalized by straggling processors, underutilizes compute
resources, and is not robust in the face of failing processors or nodes.
On the other hand, asynchronous approaches using parameter servers
create a communication bottleneck and underutilize the available network
resources, slowing convergence.

Individual researchers also have different numbers of nodes at their
disposal, including compute and network resources.
So determining the best approach for a given number of nodes, as well as the
approach that scales to the most number of nodes, is of interest to
practitioners with finite resources.

We are concerned with the following questions:
{
\renewcommand{\theenumi}{\alph{enumi}}
\begin{enumerate}
  \item How fast do asynchronous and synchronous SGD algorithms converge
    at both the beginning of training (large step sizes)
    and at the end of training (large step sizes)?
  \item How does the convergence of async.~and sync.~SGD
    vary with the number of nodes?
\end{enumerate}
}
To compare the strengths and weaknesses of asynchronous and synchronous
SGD algorithms, we train a modern ResNet convolutional network \citep{resnet}
on the ImageNet dataset \citep{ilsvrc} using various distributed SGD methods.
We primarily compare synchronous all-reduce SGD, the recently proposed
asynchronous elastic averaging SGD \citep{easgd},
as well as our own method, asynchronous gossiping SGD, based on an algorithm
originally developed in a different problem setting
\citep{Ram+Nedic+Veeravalli_2009}.
Gossiping SGD is an asynchronous method that does not use a centralized
parameter server, and
in a sense, gossiping is a decentralized version of elastic averaging.
We find that asynchronous SGD, including both elastic averaging and gossiping,
exhibits the best scaling at larger step sizes and, perhaps counterintuitively,
at smaller scales (up to around 32 distributed nodes).
For smaller step sizes and at larger scales, all-reduce consistently converges
to the most accurate solution faster than the asynchronous methods.








\section{Background}

In this section, we will describe the baseline synchronous and asynchronous
SGD methods,
as well as a recently proposed asynchronous method that is more scalable than
its predecessor.
We will use the following naming convention for SGD:
$\theta$ are the parameters over which the objective is minimized,
$\tilde\theta$ is the center parameter (if applicable),
$\alpha$ is the step size,
$\mu$ is the momentum,
subscript $i$ refers to the $i$-th node out of $p$ total nodes,
and subscript $t$ refers to the $t$-th (minibatch) iteration.
Additionally, $b$ will refer to the \emph{per-node} minibatch size,
whereas $m$ will refer to the \emph{aggregate} minibatch size summed across
all nodes.

\subsection{Synchronous All-Reduce SGD}

In traditional synchronous all-reduce SGD, there are two alternating
phases proceeding in lock-step:
(1) each node computes its local parameter gradients, and
(2) all nodes collectively communicate all-to-all to compute an aggregate
gradient, as if they all formed a large distributed minibatch.
The second phase of exchanging gradients forms a barrier and is the
communication-intensive phase, usually implemented by an eponymous
all-reduce operation.
The time complexity of an all-reduction can be decomposed into latency-bound
and bandwidth-bound terms.
Although the latency term scales with $O(\log(p))$,
there are fast ring algorithms which have bandwidth term
independent of $p$ \cite{collective}.
With modern networks capable of handling bandwidth on the order of 1--10 GB/s
combined with neural network parameter sizes on the order of 10--100 MB,
the communication of gradients or parameters between nodes across a network can
be very fast.
Instead, the communication overhead of all-reduce results from its use of a
synchronization barrier, where all nodes must wait for all other nodes until the
all-reduce is complete before proceeding to the next stochastic gradient
iteration.
This directly leads to a straggler effect where the slowest nodes will prevent
the rest of the nodes from making progress.
Examples of large-scale synchronous data parallel SGD for distributed deep
learning are given in
\citep{baidu_deepimage},
\citep{firecaffe_arxiv},
\citep{intel_sync},
and \citep{revisiting_sync}.
We provide pseudocode for synchronous data-parallel SGD in
Algorithm \ref{alg:sync-sgd}.
\begin{algorithm}
  \caption{Synchronous all-reduce SGD.}
  \label{alg:sync-sgd}
  \begin{algorithmic}
    \STATE initialize $\theta_{0,i}\gets\theta_0$
    \FOR{$t\in\{0,\ldots,T\}$}
      \STATE $\Delta\theta_{t,i} \gets -\alpha_t\nabla f_i(\theta_{t,i};X_{t,i})+\mu\Delta\theta_{t-1}$
      \STATE $\Delta\theta_{t} \gets \text{all-reduce-average}(\Delta\theta_{t,i})$
      \STATE $\theta_{t+1,i} \gets \theta_{t,i}+\Delta\theta_{t}$
    \ENDFOR
  \end{algorithmic}
\end{algorithm}

\subsection{Asynchronous Parameter-Server SGD}


A different approach to SGD consists of each node asynchronously performing
its own gradient updates and occasionally synchronizing its parameters with a
central parameter store.
This form of asynchronous SGD was popularized by ``Hogwild'' SGD
\citep{hogwild},
which considered solving sparse problems on single machine shared memory
systems.
``Downpour'' SGD \citep{downpour} then generalized the approach to distributed
SGD where nodes communicate their gradients with a central
\emph{parameter server}.
The main weakness of the asynchronous parameter-server approach to SGD is that
the workers communicate all-to-one with a central server,
and the communication throughput is limited
by the finite link reception bandwidth at the server.
One approach for alleviating the communication bottleneck is
introducing a delay between rounds of communication,
but increasing the delay greatly decreases the rate of convergence \citep{easgd}.
Large scale asynchronous SGD for deep learning was first implemented in
Google DistBelief \citep{downpour}
and has also been implemented in \citep{msr_adam};
large scale parameter server systems in the non-deep learning setting have also
been demonstrated in \citep{param_server_2013} and \citep{param_server_2014}.

\subsection{Elastic Averaging SGD}


Elastic averaging SGD \citep{easgd} is a new algorithm belonging to the family
of asynchronous parameter-server methods
which introduces a modification to the usual stochastic gradient objective to
achieve faster convergence.
Elastic averaging seeks to maximize the consensus between
the center parameter $\tilde\theta$ and the local parameters $\theta_i$
in addition to the loss:
\begin{align}
  F_\text{consensus}(\theta_1,\ldots,\theta_p,\tilde\theta)
  &=\sum_{i=1}^p \left[ f(\theta_i;X_i) + \frac{\rho}{2}\|\theta_i-\tilde\theta\|^2 \right].
\end{align}
The elastic averaging algorithm is given in Algorithm \ref{alg:elastic-sgd}.
The consensus objective of elastic averaging is closely related to the augmented
Lagrangian of ADMM, and the gradient update derived from the consensus objective
was shown by \citep{easgd} to converge significantly faster than vanilla async
SGD.
However, as elastic averaging is a member of the family of asynchronous
parameter-server approaches,
it is still subject to a communication bottleneck between the central server
and the client workers.

Because recent published results indicate that elastic averaging dominates
previous asynchronous parameter-server methods \citep{easgd},
we will only consider elastic averaging from this point on.
\begin{algorithm}
  \caption{
    Elastic averaging SGD.
    The hyperparameter $\beta$ is the moving rate.
  }
  \label{alg:elastic-sgd}
  \begin{minipage}[t]{0.5\textwidth}
    \begin{algorithmic}
      \STATE \# client code
      \STATE initialize $\theta_{0,i}\gets\theta_0$
      \FOR{$t\in\{0,\ldots,T\}$}
        \IF{$t>0$ and $t \equiv 0 \mod \tau$}
          \STATE $\tilde\theta \gets \text{receive-server-param}()$
          \STATE $\text{send-param-update}(+\beta(\theta_{t,i}-\tilde\theta))$
          \STATE $\theta_{t,i} \gets \theta_{t,i}-\beta(\theta_{t,i}-\tilde\theta)$
        \ENDIF
        \STATE $\Delta\theta_{t,i} \gets -\alpha_t\nabla f_i(\theta_{t,i};X_{t,i})+\mu\Delta\theta_{t-1,i}$
        \STATE $\theta_{t+1,i} \gets \theta_{t,i}+\Delta\theta_{t,i}$
      \ENDFOR
    \end{algorithmic}
  \end{minipage}
  \begin{minipage}[t]{0.5\textwidth}
    \begin{algorithmic}
      \STATE \# server code
      \STATE initialize $\tilde\theta_0\gets\theta_0$
      \FOR{$t\in\{0,\ldots,T\}$}
        \STATE $\text{send-server-param}(\tilde\theta_t)$
        \STATE $\Delta\theta \gets \text{receive-param-update}()$
        \STATE $\tilde\theta_{t+1} \gets \tilde\theta_t+\Delta\theta$
      \ENDFOR
    \end{algorithmic}
  \end{minipage}
\end{algorithm}

\section{Gossiping SGD}

\subsection{Algorithm}

In a nutshell, the synchronous all-reduce algorithm consists of two repeating
phases: (1) calculation of the local gradients at each node, and (2) exact
aggregation of the local gradients via all-reduce.
To derive gossiping SGD, we would like to replace the synchronous all-reduce
operation with a more asynchronous-friendly communication pattern.
The fundamental building block we use is a \emph{gossip aggregation} algorithm
\citep{Kempe+Dobra+Gehrke_2003,Boyd++_2006},
which combined with SGD leads to the gossiping SGD algorithm.
Asynchronous gossiping SGD was introduced in
\citep{Ram+Nedic+Veeravalli_2009} for the general case of a sparse
communication graph between nodes (e.g.~wireless sensor networks).
The original problem setting of gossiping also typically involved synchronous
rounds of communication, whereas we are most interested in asynchronous gossip.

The mathematical formulation of the gossiping SGD update can also be derived by
conceptually linking gossiping to elastic averaging.
Introduce a distributed version of the global consensus objective,
in which the center parameter is replaced with the average of the local
parameters:
\begin{align}
  F_\text{dist-consensus}(\theta_1,\ldots,\theta_p)
  &=\sum_{i=1}^p \left[ f(\theta_i;X_i) + \frac{\rho}{2}\left\| \theta_i-\frac{1}{p}\sum_{j=1}^p\theta_j \right\|^2 \right].
\end{align}%
%
%
The corresponding gradient steps look like the following:
\begin{align}
  \theta_{t,i}'
  &=\theta_{t,i} - \alpha\nabla f(\theta_{t,i};X_i) \\
  \theta_{t+1,i}
  &=\theta_{t,i}' - \beta\left( \theta_{t,i}'-\frac{1}{p}\sum_{j=1}^p \theta_{t,j}' \right).
\end{align}
If we replace the distributed mean $\frac{1}{p}\sum_{j=1}^p\theta_{t,j}$ with
the unbiased one-node estimator $\theta_{t,j_{t,i}}$,
such that $j_{t,i}\sim\text{Uniform}(\{1,\ldots,p\})$
and $\mathbb E[\theta_{t,j_{t,i}}]=\frac{1}{p}\sum_{j=1}^p\theta_{t,j}$,
then we derive the gossiping SGD update:
\begin{align}
  \theta_{t,i}'
  &=\theta_{t,i} - \alpha\nabla f(\theta_{t,i};X_i) \\
  \theta_{t+1,i}
  &=\theta_{t,i}' - \beta ( \theta_{t,i}'-\theta_{j_{t,i},t}' ) \\
  &=(1-\beta)\theta_{t,i}' + \beta\theta_{j_{t,i},t}'.
\end{align}
To make this more intuitive, we describe a quantity related to the distributed
consensus, the \emph{diffusion potential}.
Fix $\theta_i^{(0)}=\theta_{t,i}$, and consider the synchronous gossip setting
of purely calculating an average of the parameters and where there are no
gradient steps.
Then the updated parameter after $k$ repeated gossip rounds, $\theta_i^{(k)}$,
can be represented as a
weighted average of the initial parameter values $\theta_j^{(0)}$.
Denoting the weighted contribution of $\theta_j^{(0)}$ toward $\theta_i^{(k)}$
by $v_{i,j}^{(k)}$,
the diffusion potential is:
\begin{align}
  \Phi_\text{diffusion}(\theta_1^{(k)},\ldots,\theta_p^{(k)})
  &=\sum_{i=1}^p \sum_{j=1}^p \left\| v_{i,j}^{(k)}\theta_j^{(0)} - \frac{1}{p}\sum_{j'=1}^p v_{i,j'}^{(k)}\theta_{j'}^{(0)} \right\|^2.
\end{align}
It can be shown that repeated rounds of a form of gossiping
reduces the diffusion potential by a fixed rate per round
\citep{Kempe+Dobra+Gehrke_2003}.

If $j_{t,i}$ is chosen uniformly as above, then the algorithm is equivalent to
``pull-gossip,'' i.e.~each node pulls or receives $\theta_j$ from one and only
one other random node per iteration.
On the other hand, if we replace the ``one-node estimator'' with querying
$\theta_j$ from multiple nodes, with the constraint that each $j$ is represented
only once per iteration, then the algorithm becomes ``push-gossip,''
i.e.~each node pushes or sends its own $\theta_i$ to one and only one other
random node, while receiving from between zero and multiple other nodes.
Push-gossiping SGD can be interpreted as an interleaving of a gradient step and
a simplified push-sum gossip step
\citep{Kempe+Dobra+Gehrke_2003}.
Algorithms \ref{alg:pull-gossip-sgd} and \ref{alg:push-gossip-sgd} describe
pull-gossiping and push-gossiping SGD respectively.

\begin{minipage}[t]{0.48\textwidth}
\begin{algorithm}[H]
  \caption{Pull-gossiping SGD.}
  \label{alg:pull-gossip-sgd}
  \begin{algorithmic}
    \STATE initialize $\theta_{0,i}\gets\theta_0$
    \FOR{$t\in\{0,\ldots,T\}$}
      \IF{$t>0$ and $t \equiv 0 \mod \tau$}
        \STATE set $x_{i} \gets \theta_{t,i}$
        \STATE choose a target $j$
        \STATE $\theta_{t,i} \gets$ average of $x_{i},x_{j}$
      \ENDIF
      \STATE $\Delta\theta_{t,i} \gets -\alpha_t\nabla f_i(\theta_{t,i};X_{t,i})+\mu\Delta\theta_{t-1,i}$
      \STATE $\theta_{t+1,i} \gets \theta_{t,i}+\Delta\theta_{t,i}$
    \ENDFOR
  \end{algorithmic}
\end{algorithm}
\end{minipage}
\begin{minipage}[t]{0.48\textwidth}
\begin{algorithm}[H]
  \caption{Push-gossiping SGD.}
  \label{alg:push-gossip-sgd}
  \begin{algorithmic}
    \STATE initialize $\theta_{0,i}\gets\theta_0$
    \FOR{$t\in\{0,\ldots,T\}$}
      \IF{$t>0$ and $t \equiv 0 \mod \tau$}
        \STATE set $x_{i} \gets \theta_{t,i}$
        \STATE choose a target $j$
        \STATE send $x_{i}$ to $i$ (ourselves) and to $j$
        \STATE $\theta_{t,i} \gets$ average of received $x$'s
      \ENDIF
      \STATE $\Delta\theta_{t,i} \gets -\alpha_t\nabla f_i(\theta_{t,i};X_{t,i})+\mu\Delta\theta_{t-1,i}$
      \STATE $\theta_{t+1,i} \gets \theta_{t,i}+\Delta\theta_{t,i}$
    \ENDFOR
  \end{algorithmic}
\end{algorithm}
\end{minipage}

\subsection{Analysis}


Our analysis of gossiping SGD is based on the analyses in
\citep{nesterov_book,Boyd++_2006,Ram+Nedic+Veeravalli_2009,Touri+Nedic+Ram_2010,easgd}.
We assume that all processors are able to communicate with all other processors
at each step.
The main convergence result is the following:
\begin{theorem}
Let $f$ be a $m$-strongly convex function with $L$-Lipschitz gradients. Assume that we can sample gradients $g = \nabla f(\theta; X_i) + \xi_i$ with additive noise with zero mean $\mathbb{E}[\xi_i] = 0$ and bounded variance $\mathbb{E}[\xi_i^T \xi_i] \leq \sigma^2$. Then, running the asynchronous pull-gossip algorithm, with constant step size $0 < \alpha \leq \tfrac{2}{m+L}$, the expected sum of squares convergence of the local parameters to the optimal $\theta_\ast$ is bounded by 
\begin{align}
\mathbb{E}[\Vert \theta_t - \theta_\ast \mathbf{1} \Vert^2] \leq \left( 1 - \frac{2\alpha}{p} \frac{mL}{m+L} \right)^t \Vert \theta_0 - \theta_\ast \mathbf{1} \Vert^2 + p \alpha \sigma^2 \frac{m+L}{2mL}
\end{align}
Furthermore, with the additional assumption that the gradients are uniformly bounded as $\sup \vert \nabla f(\theta) \vert \leq C$, the expected sum of squares convergence of the local parameters to the mean $\bar{\theta}_t$ is bounded by
\begin{align}
\mathbb{E}[\Vert \theta_t - \bar{\theta}_t \mathbf{1} \Vert^2] &\leq \left( \lambda \left( 1- \frac{\alpha \, m}{p} \right) \right)^t \Vert \theta_0 - \bar{\theta}_0 \mathbf{1} \Vert^2  + \frac{\lambda \alpha^2 (C^2+\sigma^2)}{1 - \lambda \left( 1- \frac{\alpha \, m}{p} \right)} \\
\text{ where } \lambda &= 1 - \frac{2\beta(1-\beta)}{p} - \frac{2 \beta^2}{p}
\end{align}
\end{theorem}
For the proofs, please see subsection (6.1) in the supplementary material.

\section{Experiments}

\begin{figure}[!ht]
  \centering


  \begin{subfigure}[b]{0.32\textwidth}
    \includegraphics[width=\textwidth]{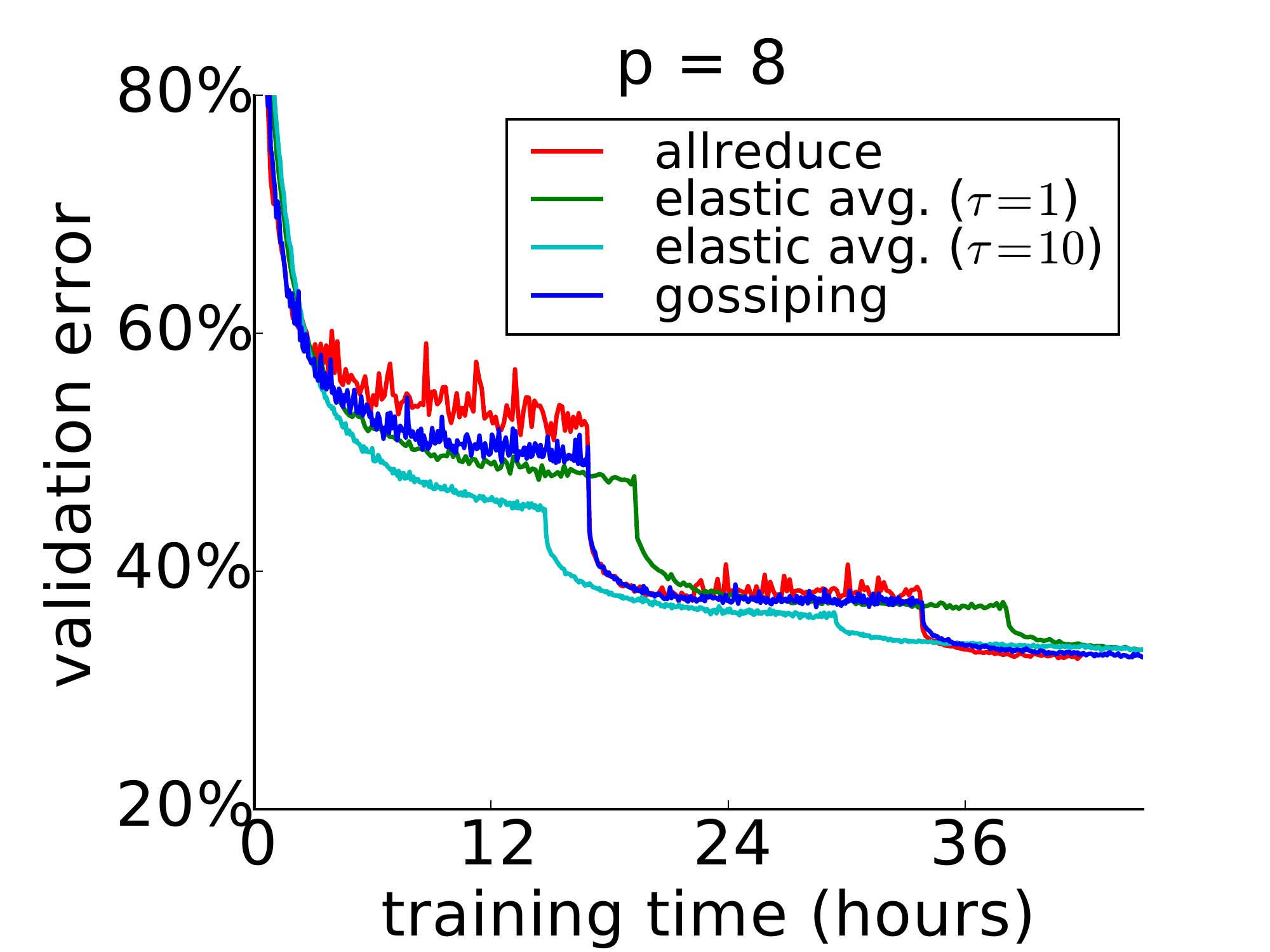}
  \end{subfigure}
  \begin{subfigure}[b]{0.32\textwidth}
    \includegraphics[width=\textwidth]{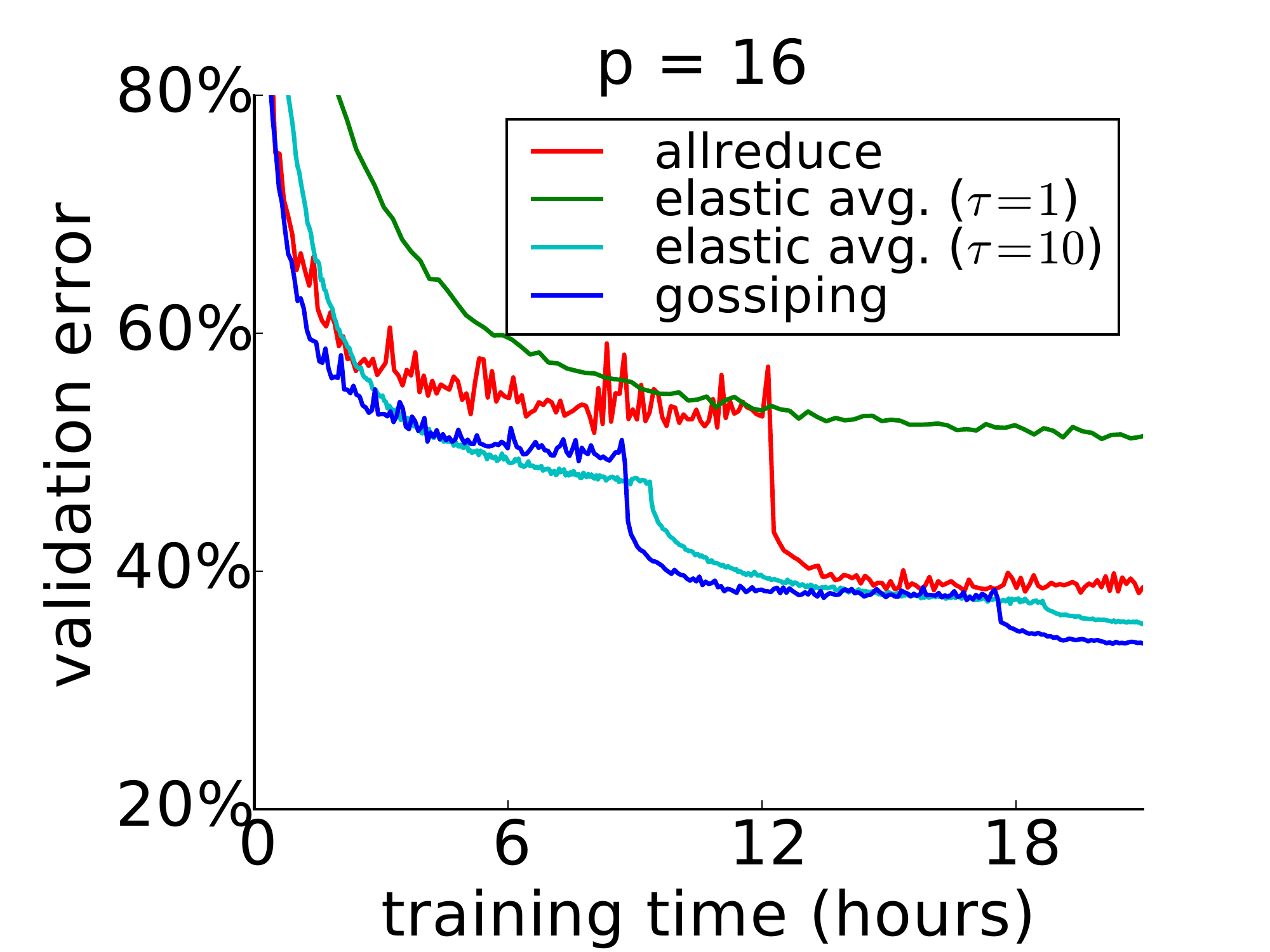}
  \end{subfigure}
  \\

  \begin{subfigure}[b]{0.32\textwidth}
    \includegraphics[width=\textwidth]{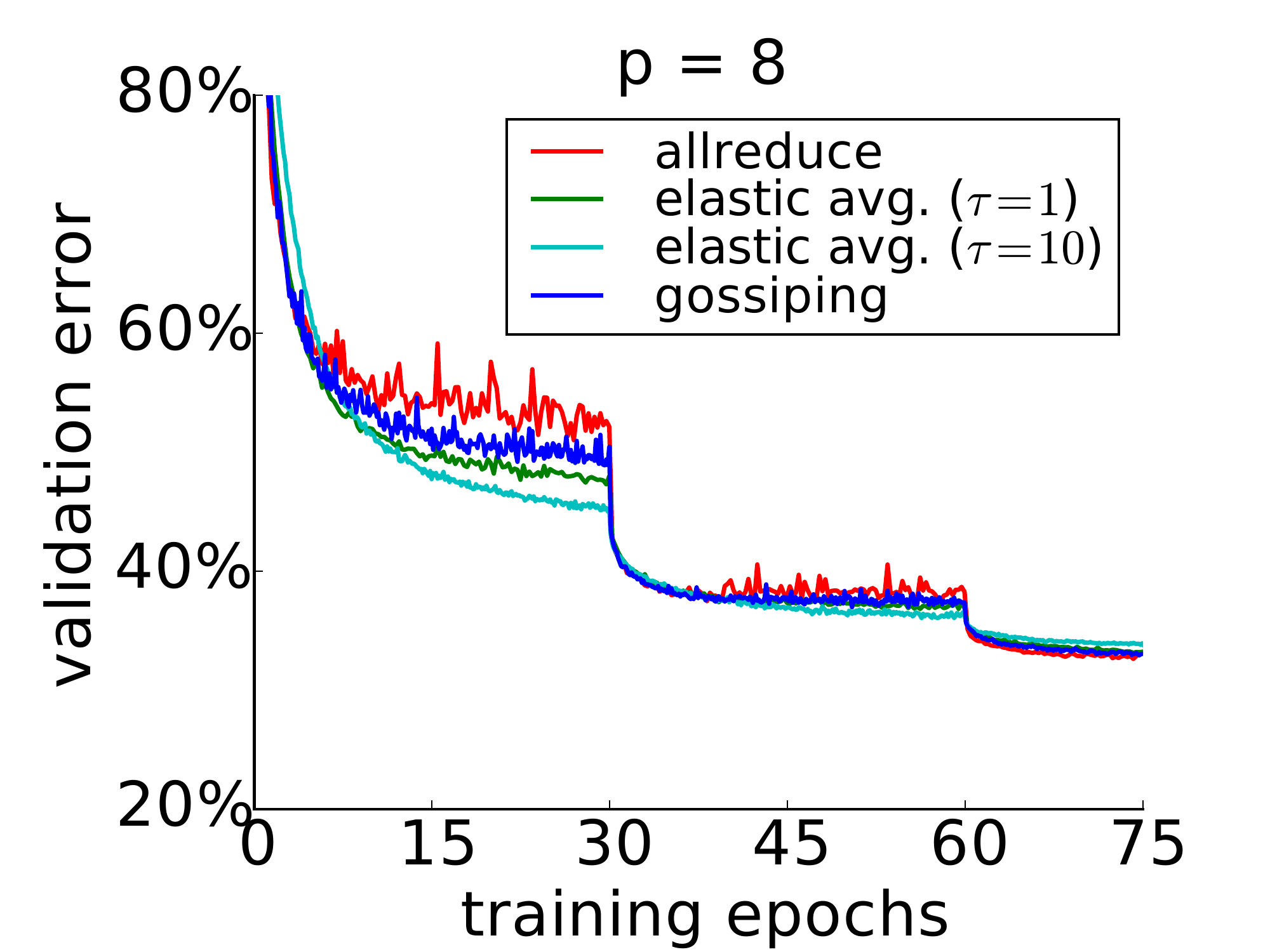}
  \end{subfigure}
  \begin{subfigure}[b]{0.32\textwidth}
    \includegraphics[width=\textwidth]{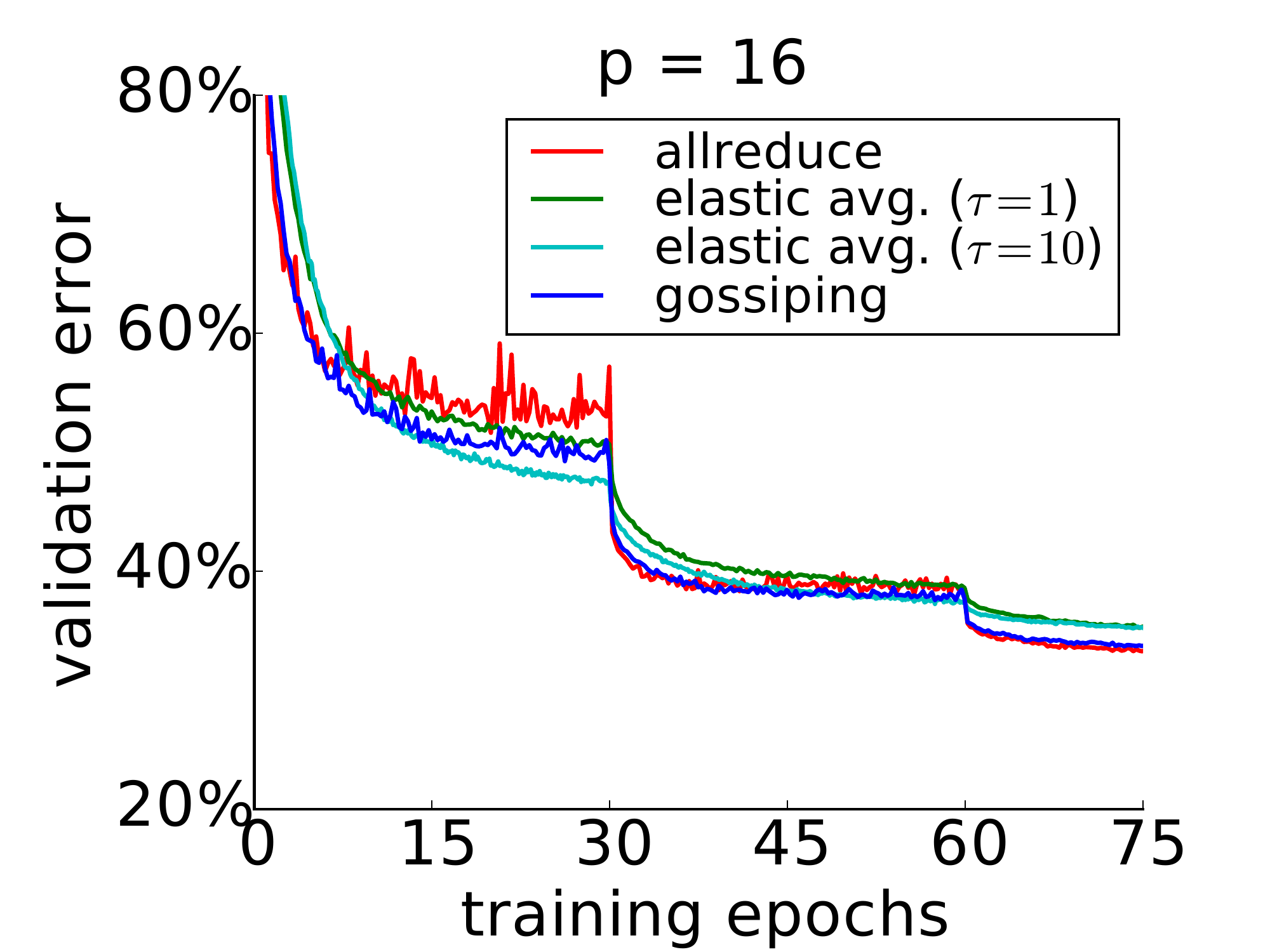}
  \end{subfigure}

  \caption{
    Center-crop validation loss and top-1 error on ImageNet
    over training wall-clock time and epochs.
    Shown are:
    (left) $p=8$ nodes with \emph{per-node} minibatch size $b=32$,
    and (right) $p=16$ nodes with \emph{per-node} minibatch size $b=16$,
  }
  \label{fig:valid_p8_p16}
\end{figure}

\begin{figure}[!ht]
  \centering


  \begin{subfigure}[b]{0.32\textwidth}
    \includegraphics[width=\textwidth]{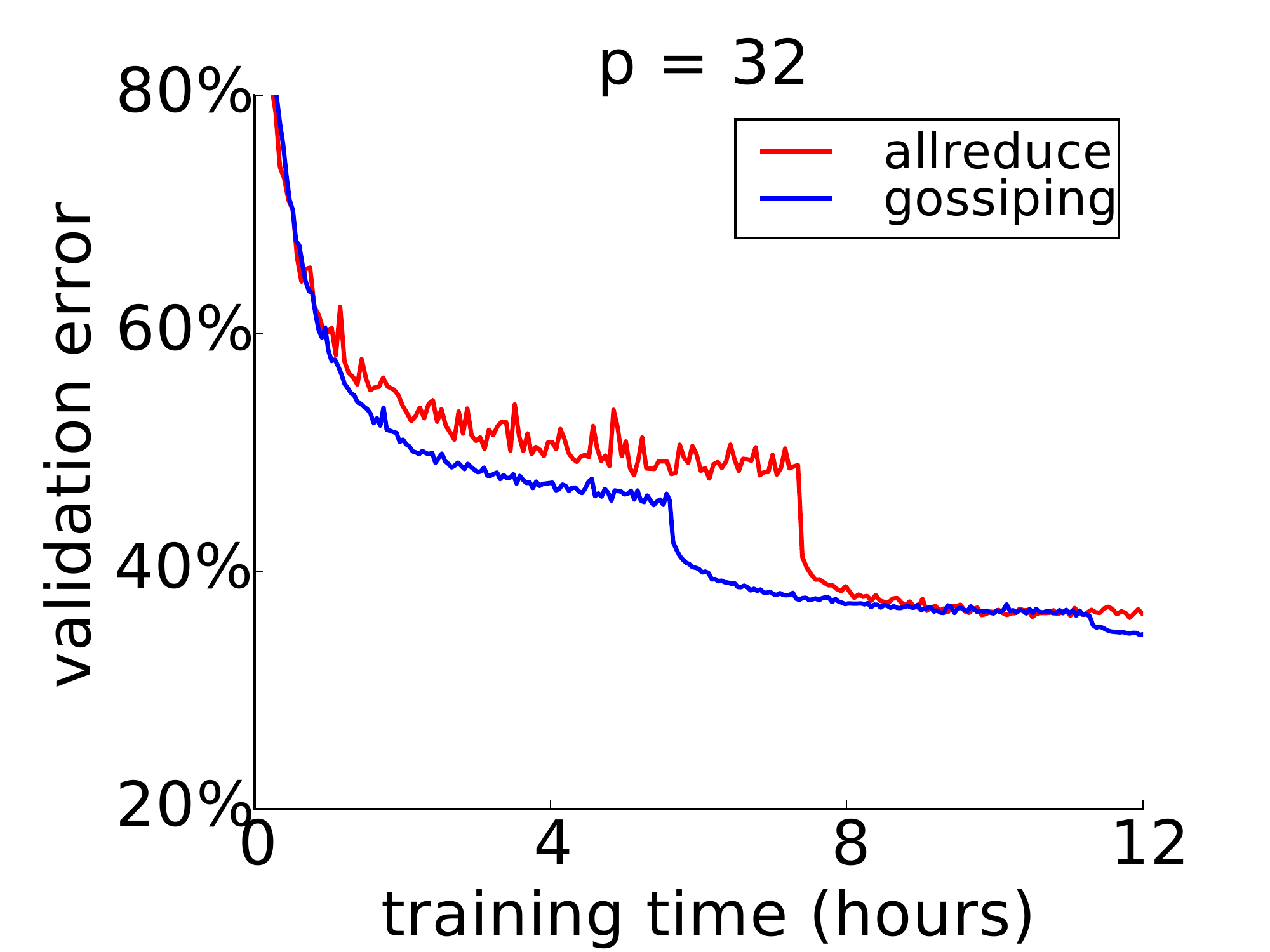}
  \end{subfigure}
  \begin{subfigure}[b]{0.32\textwidth}
    \includegraphics[width=\textwidth]{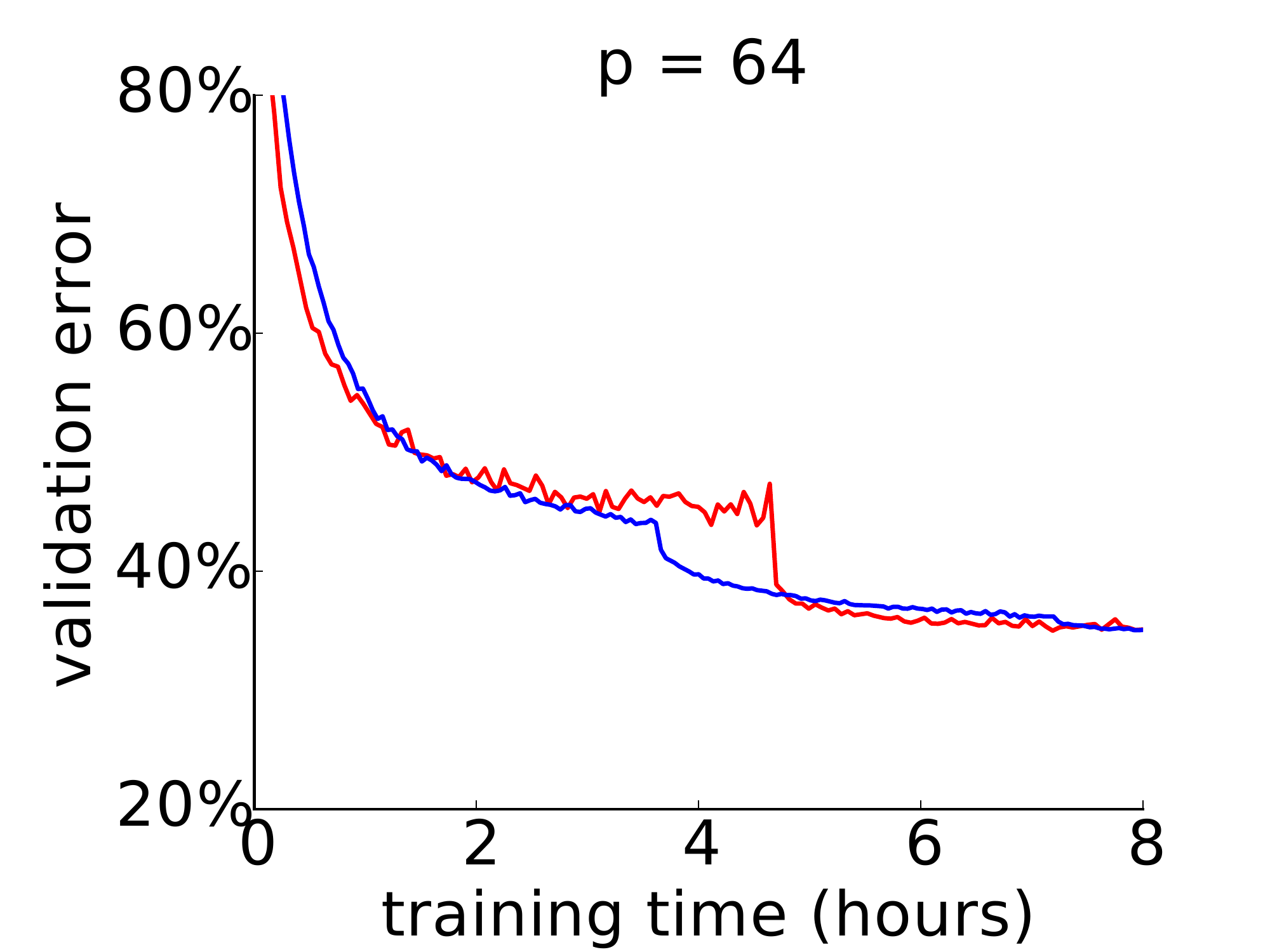}
  \end{subfigure}
  \begin{subfigure}[b]{0.32\textwidth}
    \includegraphics[width=\textwidth]{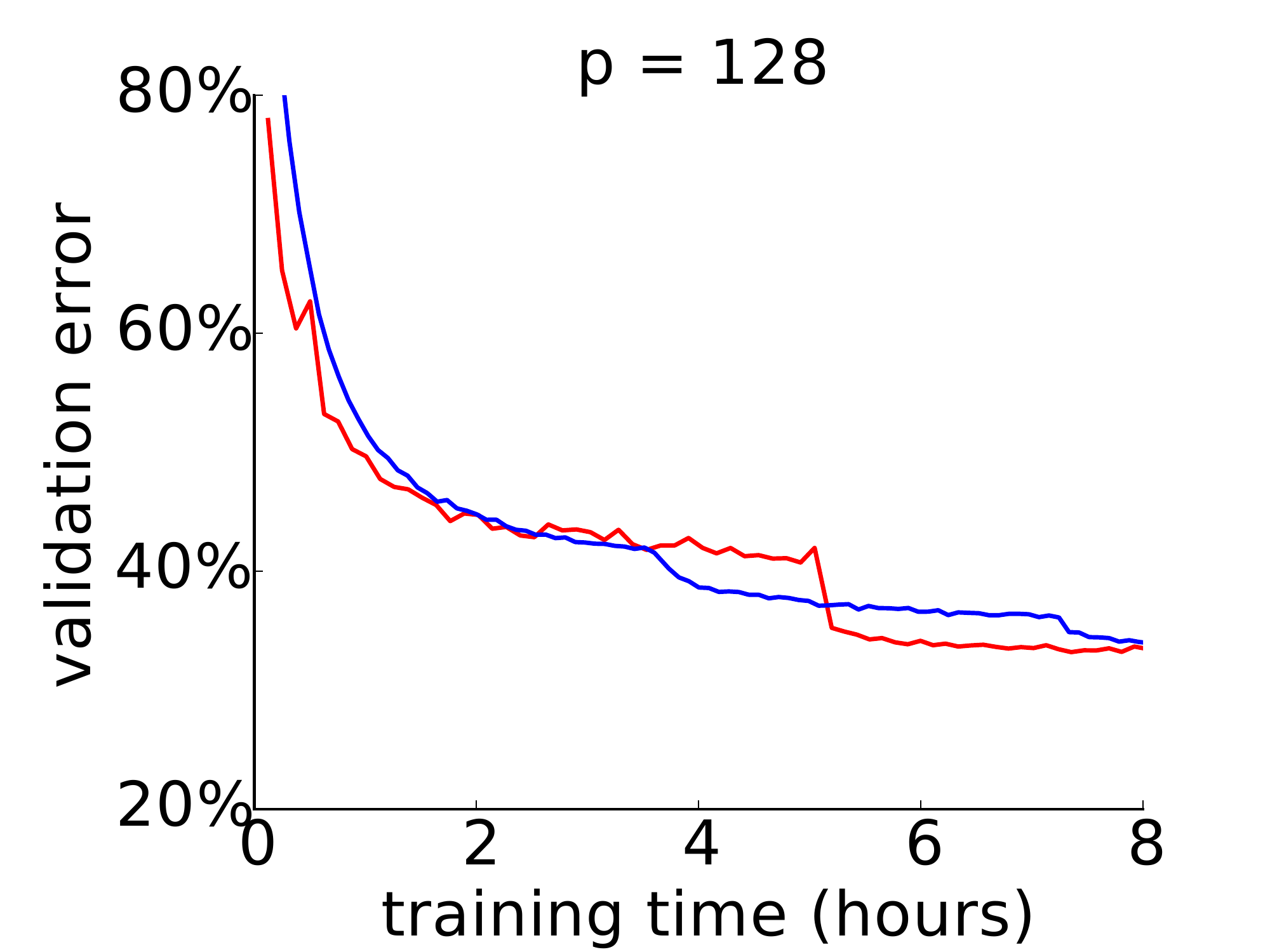}
  \end{subfigure}
  \\

  \begin{subfigure}[b]{0.32\textwidth}
    \includegraphics[width=\textwidth]{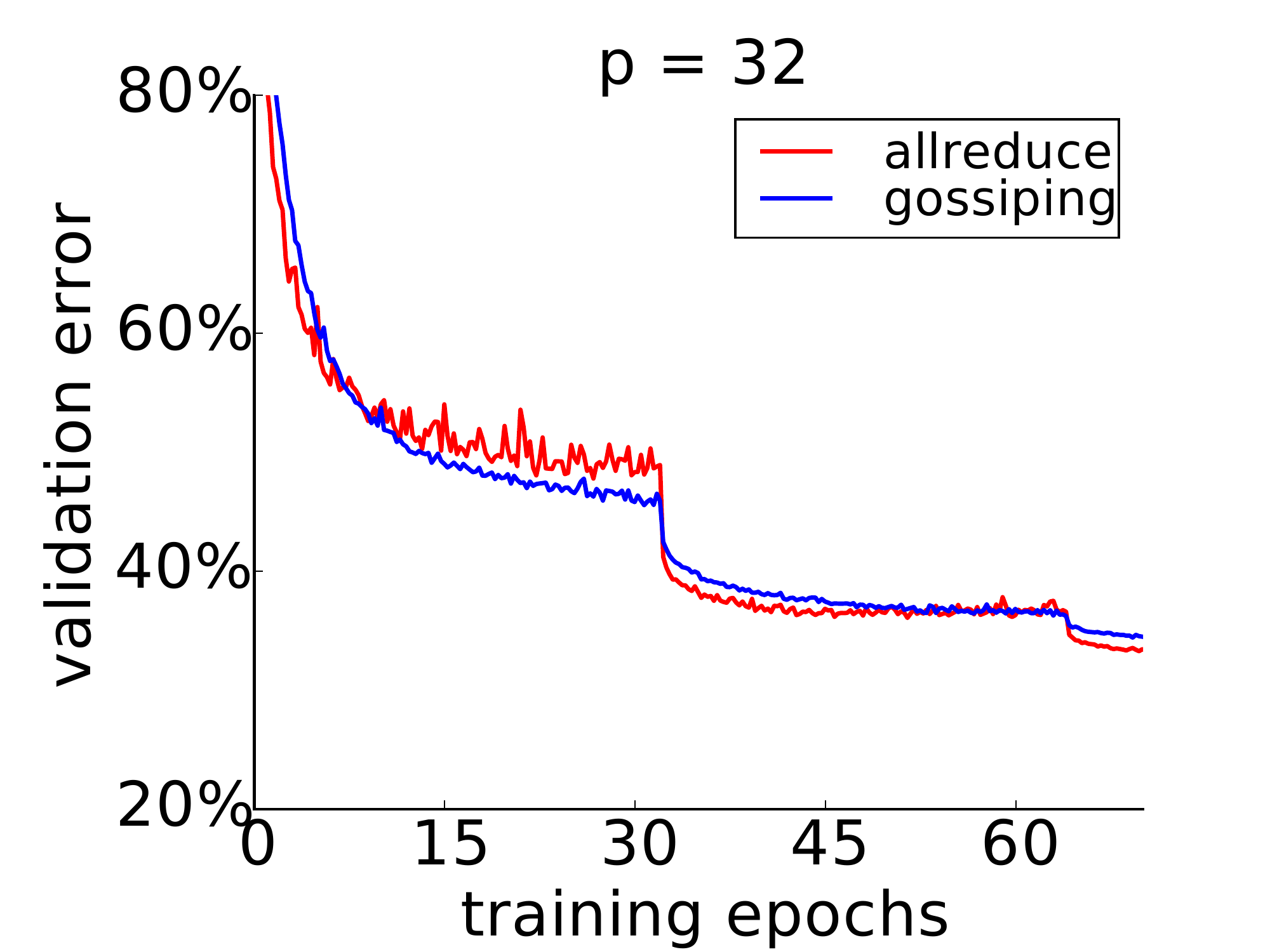}
  \end{subfigure}
  \begin{subfigure}[b]{0.32\textwidth}
    \includegraphics[width=\textwidth]{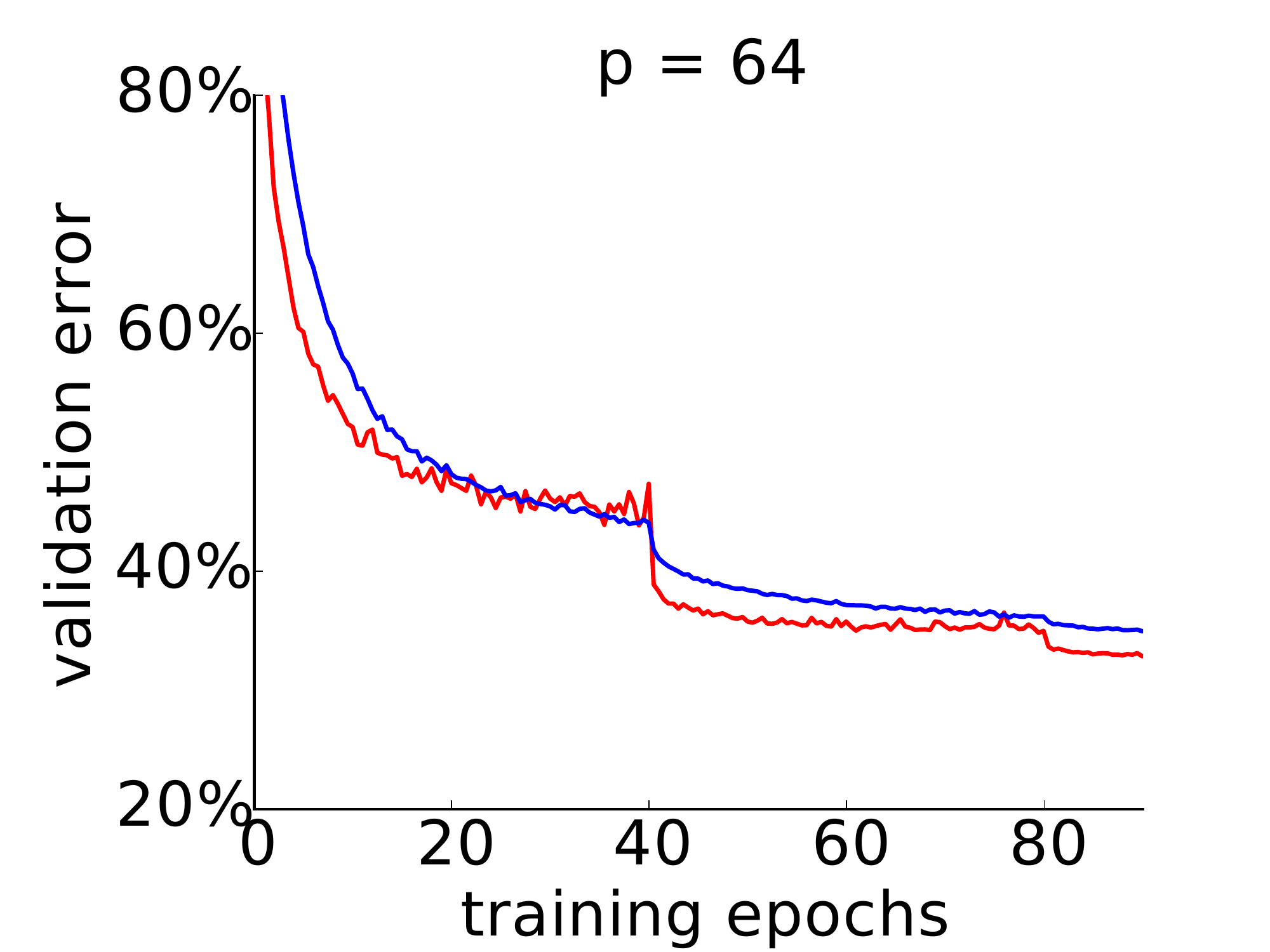}
  \end{subfigure}
  \begin{subfigure}[b]{0.32\textwidth}
    \includegraphics[width=\textwidth]{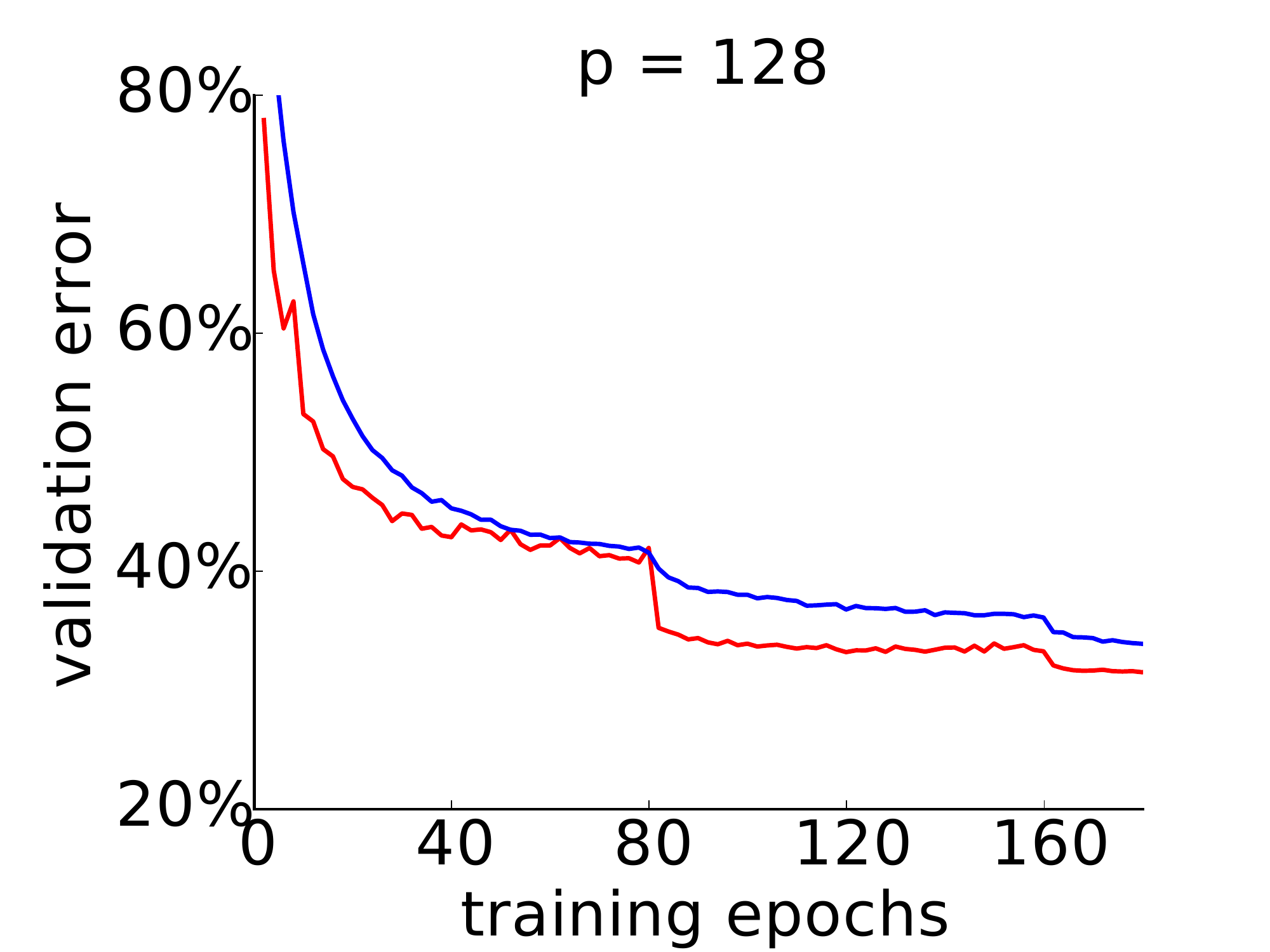}
  \end{subfigure}

  \caption{
    Center-crop validation loss and top-1 error on ImageNet
    over training wall-clock time and epochs,
    with different numbers of nodes $p$
    and \emph{per-node} minibatch size $b=16$.
    Shown are:
    (left) $p=32$ nodes,
    (middle) $p=64$ nodes,
    and (right) $p=128$ nodes.
  }
  \label{fig:valid_p32_p64}
\end{figure}

\subsection{Implementation}

We implement the communication systems of gossiping SGD and other algorithms
using Message Passing Interface (MPI) \citep{mpi}.
Because we wanted to run our code in cluster computing environments with
Infiniband or more specialized interconnects,
then targeting MPI was the easiest solution.
We targeted our code to run on GPUs, using the Nvidia CUDA 7.0 driver and
using the cuBLAS and cuDNNv4 \citep{cudnn} libraries for the core computational
kernels.

For our experiments up to $p=16$ nodes,
we use a local cluster of 16 machines, each one consisting of
an Nvidia Kepler K80 dual GPU,
an 8-core Intel Haswell E5-1680v2 CPU,
and a Mellanox ConnectX-3 FDR $4\times$ Infiniband (56 Gb/s) NIC.
We utilize only one GPU per K80.

For our larger scale experiments up to $p=128$ nodes,
we used a GPU supercomputer with over 10,000 total nodes.
Nodes consist of
an Nvidia Kepler K20X GPU
and an 8-core AMD Bulldozer Opteron 6274 CPU,
and are connected by a Cray Gemini interconnect in a 3D torus
configuration.




\subsection{Methodology}

We chose ResNets \citep{resnet} for our neural network architecture;
specifically, we trained ResNet-18, which is small enough to train rapidly for
experimentation, but also possesses features relevant to modern networks,
including depth, residual layers, and batch normalization \citep{batch_norm}.
We ran on the image classification problem of ImageNet
consisting of 1.28 million training images and 50,000 validation images
divided into 1000 classes \citep{ilsvrc}.
Our data augmentation is as follows:
we performed multi-scale training by scaling the shortest dimension of images to
between $256$ and $480$ pixels \cite{vgg},
we took random $224\times224$ crops and horizontal flips,
and we added pixelwise color noise \cite{krizhevsky}.
We evaluate validation loss and top-1 error on center crops of the validation
set images with the shortest dimension scaled to $256$ pixels.

Unless otherwise noted, we initialized the learning rate to $\alpha=0.1$,
then we annealed it twice by a factor of $0.1$.
For our experiments with aggregate minibatch size $m=pb=256$, we annealed at
exactly 150k and 300k iterations into training.
For our experiments with larger aggregate minibatch sizes, we decreased the
number of iterations at which the step size was annealed.
We used Nesterov momentum of $\mu=0.9$ and weight decay of $\lambda=10^{-4}$.
For elastic averaging, we set $\beta=0.8/p$.
For all-reduce and gossiping, we used a communication interval of $\tau=1$,
i.e.~communication occurred every iteration.
For gossiping, we used both $\tau=1$ and $\tau=10$ (the latter is recommended
in \cite{easgd}).

\subsection{Results}


Our first set of experiments compare
all-reduce, elastic averaging, and push-gossiping at $p=8$ and $p=16$
with an aggregate minibatch size $m=pb=256$.
The results are in Figure \ref{fig:valid_p8_p16}.

For $p=8$,
elastic averaging with a communication delay $\tau=10$ performs ahead of the
other methods,
Interestingly, all-reduce has practically no synchronization overhead on
the system at $p=8$ and is as fast as gossiping.
All methods converge to roughly the same minimum loss value.

For $p=16$,
gossiping converges faster than elastic averaging with $\tau=10$, and both
come ahead of all-reduce.
Additionally, elastic averaging with both $\tau=1$ and $\tau=10$ has trouble
converging to the same validation loss as the other methods
once the step size has been annealed to a small value ($\alpha=0.001$
in this case).

We also perform larger scale experiments at
$p=32$ nodes, $p=64$ nodes, and $p=128$ nodes
in the GPGPU supercomputing environment.
In this environment,
elastic averaging did not perform well so we do not show those results here;
pull-gossiping also performed better than push-gossiping, so we only show
results for pull-gossiping.
The results are in Figure \ref{fig:valid_p32_p64}.
At this scale, we begin to see the scaling advantage of synchronous all-reduce
SGD.
One iteration of gossiping SGD is still faster than one iteration of
all-reduce SGD,
and gossiping works quickly at the initial step size.
But gossiping SGD begins to converge much slower after the step size
has annealed.

We note that the training time of SGD can be thought of as the product
$(\text{wall-clock time per iteration}) \times (\text{number of iterations})$.
One observation we made consistent with \cite{revisiting_sync} was the
following:
letting synchronous all-reduce SGD run for many epochs, it will typically
converge to a lower optimal validation loss (or higher validation accuracy)
than either elastic averaging or gossiping SGD.
We found that letting all-reduce SGD run for over 1 million
iterations with a minibatch size of 256 led to a peak top-1 validation accuracy
of $68.7\%$.
However, elastic averaging often had trouble breaking $67\%$,
as did gossiping when the number of nodes was greater than $p=32$.
In other words, at larger scales the asynchronous methods require more
iterations to convergence despite lower wall-clock time per iteration.

\section{Discussion}

Revisiting the questions we asked in the beginning:
{
\renewcommand{\theenumi}{\alph{enumi}}
\begin{enumerate}
  \item
    \emph{How fast do asynchronous and synchronous SGD algorithms converge
    at both the beginning of training (large step sizes)
    and at the end of training (large step sizes)?}

    Up to around 32 nodes, asynchronous SGD can converge faster
    than all-reduce SGD when the step size is large.
    When the step size is small (roughly $0.001$ or less), gossiping can
    converge faster than elastic averaging, but all-reduce SGD converges
    most consistently.



  \item
    \emph{How does the convergence behavior of asynchronous and synchronous
    distributed SGD vary with the number of nodes?}

    Both elastic averaging and gossiping seem to converge faster than
    synchronous all-reduce SGD with fewer nodes (up to 16--32 nodes).
    With more nodes (up to a scale of 100 nodes),
    all-reduce SGD can consistently converge to a high-accuracy solution,
    whereas asynchronous methods seem to plateau at lower accuracy.
    In particular, the fact that gossiping SGD does not scale as well as does
    synchronous SGD with more nodes suggests that the asynchrony and the
    pattern of communication,
    rather than the amount of communication (both methods have low amounts of
    communication), are responsible for the difference in convergence.
\end{enumerate}
}

In this work, we focused on comparing the scaling of synchronous and
asynchronous SGD methods on a supervised learning problem on two platforms:
a local GPU cluster and a GPU supercomputer.
However, there are other platforms that are relevant for researchers,
depending on what resources they have available.
These other platforms include multicore CPUs, multi-GPU single servers,
local CPU clusters, and cloud CPU/GPU instances,
and we would expect to observe somewhat different results compared to the
platforms tested in this work.

While our experiments were all in the setting of supervised learning
(ImageNet image classification),
the comparison between synchronous and asynchronous parallelization of
learning may differ in other settings;
c.f.~recent results on asynchronous methods in deep reinforcement learning
\cite{async_rl}.
Additionally, we specifically used \emph{convolutional} neural networks
in our supervised learning experiments,
which because of their high arithmetic intensity (high ratio of floating point
operations to memory footprint)
have a different profile from networks with many fully connected operations,
including most recurrent networks.

Finally, we exclusively looked at SGD with Nesterov momentum as the underlying
algorithm to be parallelized/distributed.
Adaptive versions of SGD such as RMSProp \cite{rmsprop} and Adam \citep{adam}
are also widely used in deep learning, and their corresponding distributed
versions may have additional considerations (e.g.~sharing the squared gradients
in \cite{async_rl}).

\section*{Acknowledgments}

We thank Kostadin Ilov, Ed F.~D'Azevedo, and Chris Fuson for helping
make this work possible,
as well as Josh Tobin for insightful discussions.
We would also like to thank the anonymous reviewers for their constructive
feedback.
This research is partially funded by DARPA Award Number HR0011-12-2-0016
and by ASPIRE Lab industrial sponsors and affiliates
Intel, Google, Hewlett-Packard, Huawei, LGE, NVIDIA, Oracle, and Samsung.
This research used resources of the Oak Ridge Leadership Computing Facility at
the Oak Ridge National Laboratory, which is supported by the Office of Science
of the U.S. Department of Energy under Contract No. DE-AC05-00OR22725.

\section*{References}

\bibliographystyle{unsrtnat}
{\small\bibliography{draft}}

\newpage
\section{Appendix}

\subsection{Analysis}

The analysis below loosely follow the arguments presented in, and use a combination of techniques appearing from [17,16,10,18,9]. \\

For ease of exposition and notation, we focus our attention on the case of univariate strongly convex function $f$ with Lipschitz gradients. (Since the sum of squares errors are additive in vector components, the arguments below generalize to the case of multivariate functions.) We assume that the gradients of the function $f$ can be sampled by all processors independently up to additive noise with zero mean and bounded variance. (Gaussian noise satisfy these assumptions, for example.) We also assume a fully connected network where all processors are able to communicate with all other processors. Additional assumptions are introduced below as needed. \\

We denote by $\theta_t = \begin{bmatrix} \theta_{t,1}, \cdots, \theta_{t,p} \end{bmatrix}^T$ the vector containing all local parameter values at time step $t$, we denote by $\theta_\ast$ the optimal value of the objective function $f$, and we denote by $\bar{\theta}_t \equiv \tfrac{1}{p} \sum_{i=1}^p \theta_{t,i}$ the spatial average of the local parameters taken at time step $t$. \\

We derive bounds for the following two quantities:

\begin{itemize}
	\item $\mathbb{E}[\Vert \theta_t - \theta_\ast \mathbf{1} \Vert^2]$, the squared sum of the local parameters' deviation from the optimum $\theta_\ast$ \\
	\item $\mathbb{E}[\Vert \theta_t - \bar{\theta}_t \mathbf{1} \Vert^2]$, the squared sum of the local parameters' deviation from the mean $\bar{\theta}_t = \frac{1}{p} \sum_{i=1}^p \theta_{t,i}$
\end{itemize}
where the expectation is taken with respect to both the ``pull'' parameter choice and the gradient noise term. In the literature [18], the latter is usually referred to as ``agent agreement'' or ``agent consensus''.  

\subsubsection{Synchronous pull-gossip algorithm}

We begin by analyzing the synchronous version of the pull-gossip algorithm described in Algorithm 3. For each processor $i$, let $j_i$ denote the processor, chosen uniformly randomly from $\{1, \cdots, p\}$ from which processor $i$ ``pulls'' parameter values. The update for each $\theta_{t,i}$ is given by

\begin{align}
\theta_{t+1,i} = \frac{1}{2} \left( \theta_{t,i} + \theta_{t, j_i} \right) - \alpha_{t} \left( \nabla f \left( \frac{1}{2} \left( \theta_{t,i} + \theta_{t, j_i} \right) ; X_{t,i} \right) + \xi_{t, i} \right)
\end{align} 

We prove the following rate of convergence for the synchronous pull-gossip algorithm. 

\begin{theorem}
Let $f$ be a $m$-strongly convex function with $L$-Lipschitz gradients. Assume that we can sample gradients $g = \nabla f(\theta; X_i) + \xi_i$ with additive noise with zero mean $\mathbb{E}[\xi_i] = 0$ and bounded variance $\mathbb{E}[\xi_i^T \xi_i] \leq \sigma^2$. Then, running the synchronous pull-gossip algorithm as outlined above with step size $0 < \alpha \leq \tfrac{2}{m+L}$, the expected sum of squares convergence of the local parameters to the optimal $\theta_\ast$ is bounded by 

\begin{align}
\mathbb{E}[ \Vert \theta_t - \theta_\ast \mathbf{1} \Vert^2 ] &\leq \left( 1 - 2 \alpha \frac{mL}{m+L} \right)^t \Vert \theta_0 - \theta_\ast \mathbf{1} \Vert^2 + p \alpha \sigma^2\frac{m+L}{2mL}
\end{align}
\end{theorem}

\begin{remark}
Note that this bound is characteristic of SGD bounds, where the iterates converge to within some ball around the optimal solution. It can be shown by induction that a decreasing step size schedule of $\alpha_t \sim \mathcal{O}(1/pt)$ can be used to achieve a convergence rate of $\mathcal{O}(1/t)$.
\end{remark}

\begin{proof}
For notational simplicity, we denote $\theta_{t, i, j_i} \equiv \frac{1}{2} \left( \theta_{t,i} + \theta_{t, j_i} \right)$, and drop the $X_{t,i}$ in the gradient term. We tackle the first quantity by conditioning on the previous parameter values and expanding as  
\begin{align}
\mathbb{E}[ \Vert \theta_{t+1} - \theta_\ast \mathbf{1} \Vert^2 \, \vert \, \theta_{t} ] &= \mathbb{E}\left[ \sum_{i=1}^p \left( \theta_{t,i,j_i} - \theta_\ast \right)^T \left( \theta_{t,i,j_i} - \theta_\ast \right) \, \Big\vert \, \theta_{t} \right] \label{eq:firstterm} \\
&- 2 \alpha_{t} \mathbb{E}\left[ \sum_{i=1}^p \nabla f(\theta_{t, i, j_i})^T (\theta_{t, i, j_i} - \theta_\ast) \, \Big\vert \, \theta_{t} \right] \label{eq:term2} \\
&- 2 \alpha_{t} \mathbb{E}\left[ \sum_{i=1}^p \xi_{t,i}^T \left( \theta_{t,i,j_i} - \theta_\ast \right) \, \Big\vert \, \theta_{t} \right]  \\
&+ \alpha_{t}^2 \left[ \sum_{i=1}^p \left( \nabla f(\theta_{t, i, j_i}) + \xi_{t, i} \right)^T \left( \nabla f(\theta_{t, i, j_i}) + \xi_{t, i} \right) \, \Big\vert \, \theta_{t} \right] \label{eq:lastterm}
\end{align}

Recalling that strongly convex functions satisfy [17], $\forall x, z$,
\begin{align} 
& (\nabla f(x) - \nabla f(z))^T (x - z) \\
\nonumber & \geq \frac{mL}{m+L} (x - z)^T(x - z) + \frac{1}{m+L} (\nabla f(x) - \nabla f(z))^T (\nabla f(x) - \nabla f(z)) && \label{eq:stronglyconvexssumption}
\end{align}
we can use this inequality, with $x = \theta_{t, i, j_i}$ and $z = \theta_\ast$ to bound the term in \eqref{eq:term2}:
\begin{align}
& -2\alpha_{t} \mathbb{E} \left[ \sum_{i=1}^p \nabla f(\theta_{t, i, j_i})^T (\theta_{t, i, j_i} - \theta_\ast) \, \Big\vert \, \theta_{t} \right] \\
\nonumber & \leq - \frac{2 \alpha_{t} mL}{m+L} \mathbb{E}\left[ \sum_{i=1}^p (\theta_{t,i,j_i} - \theta_\ast)^T (\theta_{t,i,j_i} - \theta_\ast) \, \Big\vert \, \theta_{t} \right] - \frac{2 \alpha_{t} }{m+L} \mathbb{E}\left[ \sum_{i=1}^p \nabla f(\theta_{t,i,j_i})^T \nabla f(\theta_{t, i, j_i}) \, \Big\vert \, \theta_{t} \label{eq:stronglyconvex} \right]
\end{align}

Using \eqref{eq:stronglyconvex}, and regrouping terms in \eqref{eq:firstterm}-\eqref{eq:lastterm}, we obtain 
\begin{align}
\mathbb{E}[ \Vert \theta_{t+1} - \theta_\ast \mathbf{1} \Vert^2 \, \vert \, \theta_{t} ] &\leq \left( 1 - 2 \alpha_{t} \frac{mL}{m+L} \right) \mathbb{E}\left[ \sum_{i=1}^p \left( \theta_{t,i,j_i} - \theta_\ast \right)^T \left( \theta_{t,i,j_i} - \theta_\ast \right) \, \Big\vert \, \theta_{t} \right] \\
&+ \left( \alpha_{t}^2 - 2 \alpha_{t} \frac{1}{m+L} \right) \mathbb{E}\left[ \sum_{i=1}^p \nabla f(\theta_{t, i, j_i})^T \nabla f(\theta_{t, i, j_i}) \, \Big\vert \, \theta_{t} \right] \label{eq:dropifalphasmall} \\
&+ \alpha^2_{t} \mathbb{E}\left[ \sum_{i=1}^p \xi_{t,i}^T \xi_{t,i} \, \Big\vert \, \theta_{t} \right] \label{eq:noiseterm}
\end{align}
In the above expression, we have dropped the terms linear in $\xi_{t,i}$, using the assumption that these noise terms vanish in expectation. In addition, if the step size parameter $\alpha$ is chosen sufficiently small, $0 < \alpha_{t} \leq \frac{2}{m+L}$, then the second term in \eqref{eq:dropifalphasmall} can also be dropped. The expression we must contend with is
\begin{align}
\mathbb{E}[ \Vert \theta_{t+1} - \theta_\ast \mathbf{1} \Vert^2 \, \vert \, \theta_{t} ] &\leq \left( 1 - 2 \alpha_{t} \frac{mL}{m+L} \right) \mathbb{E}\left[ \sum_{i=1}^p \left( \theta_{t,i,j_i} - \theta_\ast \right)^T \left( \theta_{t,i,j_i} - \theta_\ast \right) \, \Big\vert \, \theta_{t} \right] + \alpha^2_t p \sigma^2 \label{eq:recursionsync}
\end{align}

Using the definition of $\theta_{t,i,j_i}$, we can verify that the following matrix relation holds.
\begin{align}
\begin{bmatrix}
\theta_{t,1,j_1} - \theta_\ast \\ \vdots \\ \theta_{t,i,j_i} - \theta_\ast \\ \vdots \\ \theta_{t,p,j_p} - \theta_\ast
\end{bmatrix} &= \begin{bmatrix} \tfrac{1}{2} & & & \tfrac{1}{2} & \\ & \ddots & & & \\ & \tfrac{1}{2} & \tfrac{1}{2} & & & \\ & & & \ddots & \\ \tfrac{1}{2} & & & & \tfrac{1}{2} \end{bmatrix}
\begin{bmatrix} \theta_{t,1} - \theta_\ast \\ \vdots \\ \theta_{t,i} - \theta_\ast \\ \vdots \\ \theta_{t,p} - \theta_\ast \end{bmatrix} \equiv M_I (\theta_t - \theta_\ast \mathbf{1})
\end{align}
where the random matrix $M_I$ depends on the random index set $I = \{j_i\}_{i=1}^p$ of uniformly randomly drawn indices. $M_I$ has two entries in each row, one on the diagonal, and one appearing in the $j_i$th column, and is a right stochastic matrix but need not be doubly stochastic. \\

We can express this matrix as
\begin{align*}
M_I &= \frac{1}{2} \sum_{i=1}^p e_i (e_i + e_{j_i})^T
\end{align*}
and compute its second moment
\begin{align}
\mathbb{E}[M_I^T M_I] &= \frac{1}{4} \, \mathbb{E}\left[ \left( I + \sum_{i=1}^p e_i e_{j_i}^T \right)^T \left( I + \sum_{i=1}^p e_i e_{j_i}^T \right) \right] \\
&= \frac{1}{2} \left( I + \frac{1}{p} \mathbf{1} \mathbf{1}^T \right) \\
&= Q \begin{bmatrix} 1 & & & \\ & \tfrac{1}{2} & & \\ & & \ddots & \\ & & & \tfrac{1}{2} \end{bmatrix} Q^T \label{eq:eigendecomp}
\end{align}
where in the last line, the orthogonal diagonalization reveals that the eigenvalues of this matrix are bounded by $\frac{1}{2} \leq \lambda_i \leq 1$. \\

Using \eqref{eq:eigendecomp}, we can further simply \eqref{eq:recursionsync} to
\begin{align}
\mathbb{E}[ \Vert \theta_{t+1} - \theta_\ast \mathbf{1} \Vert^2 \, \vert \, \theta_{t} ] &\leq \left( 1 - 2 \alpha_{t} \frac{mL}{m+L} \right) (\theta_t - \theta_\ast \mathbf{1})^T \mathbb{E}\left[ M_I^T M_I \right] (\theta_t - \theta_\ast \mathbf{1}) + \alpha^2_t p \sigma^2 \\
&\leq \left( 1 - 2 \alpha_{t} \frac{mL}{m+L} \right) \Vert \theta_t - \theta_\ast \mathbf{1} \Vert^2 + \alpha^2_t p \sigma^2
\end{align}
Assuming a constant step size $\alpha_t \equiv \alpha$, the above recursion can be unrolled to derive the bound
\begin{align}
\mathbb{E}[ \Vert \theta_t - \theta_\ast \mathbf{1} \Vert^2 ] &\leq \left( 1 - 2 \alpha \frac{mL}{m+L} \right)^t \Vert \theta_0 - \theta_\ast \mathbf{1} \Vert^2 + p \alpha \sigma^2\frac{m+L}{2mL}
\end{align}
We note that the above bound is characteristic of SGD bounds, where the iterates converge to a ball around the optimal solution, whose radius now depends on the number of processors $p$, in addition to the step size $\alpha$ and the variance of the gradient noise $\sigma^2$. \\

\end{proof}

\subsubsection{Asynchronous pull-gossip algorithm}
We provide similar analysis for the asynchronous version of the pull-gossip algorithm. As is frequently done in the literature, we model the time steps as the ticking of local clocks governed by Poisson processes. More precisely, we assume that each processor has a clock which ticks with a rate $1$ Poisson process. A master clock which ticks whenever a local processor clock ticks is then governed by a rate $p$ Poisson process, and a time step in the algorithm is defined as whenever the master clock ticks. Since each master clock tick corresponds to the tick of some local clock on processor $i$, this in turn marks the time step at which processor $i$ ``pulls'' the parameter values from the uniformly randomly chosen processor $j_i$. Modeling the time steps by Poisson processes provide nice theoretical properties, i.e. the inter-tick time intervals are i.i.d. exponential variables of rate $p$, and the local clock $i$ which causes each master clock tick is i.i.d. drawn from $\{1, \cdots, p\}$, to name a few. For an in depth analysis of this model, and results that relate the master clock ticks to absolute time, please see [16]. \\

The main variant implemented is the pull-gossip with fresh parameters in subsection 6.2.2. The update at each time step is given by
\begin{align*}
&\left\{
\begin{array}{ll}
\theta_{t+1, i} = (1-\beta) \left( \theta_{t,i} - \alpha_t \left(\nabla f(\theta_{t,i}) + \xi_{t,i} \right) \right) + \beta \theta_{t, j_i} & \\
\theta_{t+1, k} = \theta_{t, k} & \, \text{ , for } \, k \neq i
\end{array}
\right. \\
&\text{where } \xi_{t,i} \text{ is the gradient noise } 
\end{align*}
Note in the implementation we use $\beta = 1/2$, however in the analysis we retain the $\beta$ parameter for generality. \\

We have the following convergence result. 

\begin{theorem}
Let $f$ be a $m$-strongly convex function with $L$-Lipschitz gradients. Assume that we can sample gradients $g = \nabla f(\theta; X_i) + \xi_i$ with additive noise with zero mean $\mathbb{E}[\xi_i] = 0$ and bounded variance $\mathbb{E}[\xi_i^T \xi_i] \leq \sigma^2$. Then, running the asynchronous pull-gossip algorithm with the time model as described, with constant step size $0 < \alpha \leq \tfrac{2}{m+L}$, the expected sum of squares convergence of the local parameters to the optimal $\theta_\ast$ is bounded by 
\begin{align}
\mathbb{E}[\Vert \theta_t - \theta_\ast \mathbf{1} \Vert^2] \leq \left( 1 - \frac{2\alpha}{p} \frac{mL}{m+L} \right)^t \Vert \theta_0 - \theta_\ast \mathbf{1} \Vert^2 + p \alpha \sigma^2 \frac{m+L}{2mL}
\end{align}
Furthermore, with the additional assumption that the gradients are uniformly bounded as $\sup \vert \nabla f(\theta) \vert \leq C$, the expected sum of squares convergence of the local parameters to the mean $\bar{\theta}_t$ is bounded by
\begin{align}
\mathbb{E}[\Vert \theta_t - \bar{\theta}_t \mathbf{1} \Vert^2] &\leq \left( \lambda \left( 1- \frac{\alpha \, m}{p} \right) \right)^t \Vert \theta_0 - \bar{\theta}_0 \mathbf{1} \Vert^2  + \frac{\lambda \alpha^2 (C^2+\sigma^2)}{1 - \lambda \left( 1- \frac{\alpha \, m}{p} \right)} \\
\text{ where } \lambda &= 1 - \frac{2\beta(1-\beta)}{p} - \frac{2 \beta^2}{p}
\end{align}
\end{theorem}
\begin{remark}
Note again that this bound is characteristic of SGD bounds, with the additional dependence on $p$, the number of processors.
\end{remark}

\begin{proof}
For simplicity of notation, we denote $g_{t,i} \equiv \nabla f(\theta_{t,i}) + \xi_{t,i}$. We can write the asynchronous iteration step in matrix form as 
\begin{align}
\begin{bmatrix} \theta_{t+1,1} - \theta_\ast \\ \vdots \\ \theta_{t+1,i} - \theta_\ast \\ \vdots \\ \theta_{t+1,j_i} - \theta_\ast \\ \vdots \\ \theta_{t+1,p} - \theta_\ast \end{bmatrix}
&= \begin{bmatrix} 1 & & & & & & \\ & \ddots & & & & & \\ & & 1-\beta & & \beta & & \\ & & & \ddots & & & \\ & & & & 1 & & \\ & & & & & \ddots & \\ & & & & & & 1 \end{bmatrix} 
\left( \begin{bmatrix} \theta_{t,1} - \theta_\ast \\ \vdots \\ \theta_{t,i} - \theta_\ast \\ \vdots \\ \theta_{t,j_i} - \theta_\ast \\ \vdots \\ \theta_{t,p} - \theta_\ast \end{bmatrix} - \alpha_t \begin{bmatrix} 0 \\ \vdots \\ g_{t,i} \\ \vdots \\ 0 \\ \vdots \\ 0 \end{bmatrix} \right) \\
&\equiv D_{i,j_i} \left( (\theta_t - \theta_\ast) - \alpha_t g_{t} \right) \label{eq:defofD}
\end{align}
The random matrix $D_{i,j_i}$ depends on the indices $i$ and $j_i$, both of which are uniformly randomly drawn from $\{1, \cdots, p\}$. For notational convenience we will drop the subscripts, but we keep in mind that the expectation below is taken with respect to the randomly chosen indices. We can express the matrix as 
\begin{align}
D &= I + \beta e_i(e_{}j_i - e_i)^T 
\end{align}
and compute its second moment as
\begin{align}
\mathbb{E}[D^TD] &= \left( 1 - \frac{2 \beta (1-\beta)}{p} \right) \mathbf{I} + \frac{2 \beta (1-\beta)}{p^2} \mathbf{1} \mathbf{1}^T \\
&= Q \begin{bmatrix} 1 & & & \\ & 1-\frac{2\beta(1-\beta)}{p} & & \\ & & \ddots & \\ & & & 1-\frac{2\beta(1-\beta)}{p} \end{bmatrix} Q^T \label{eq:asyncdiag}
\end{align}
The orthogonal diagonalization reveals that the eigenvalues of this matrix are bounded by $1-\tfrac{2 \beta(1-\beta)}{p} \leq \lambda_i \leq 1$. \\

Using \eqref{eq:asyncdiag}, we can expand and bound the expected sum of squares deviation by
\begin{align}
&\mathbb{E}[\Vert \theta_{t+1} - \theta_\ast \mathbf{1} \Vert^2 \vert \theta_t] \nonumber \\
&= \mathbb{E} \left[ ((\theta_t - \theta_\ast \mathbf{1}) - \alpha_t g_{t})^T D^TD ((\theta_t - \theta_\ast \mathbf{1}) - \alpha_t g_{t}) \right] \\
&\leq \mathbb{E} \left[ \Vert (\theta_t - \theta_\ast \mathbf{1}) - \alpha_t g_{t} \Vert^2 \right] \\
&= \Vert \theta_t - \theta_\ast \mathbf{1} \Vert^2 - \frac{2 \alpha_t}{p} \sum_{i=1}^p \nabla f(\theta_{t,i})^T (\theta_{t,i} - \theta_\ast) + \frac{\alpha_t^2}{p} \sum_{i=1}^p \nabla f(\theta_{t,i})^T \nabla f(\theta_{t,i}) + \frac{\alpha_t^2}{p} \sum_{i=1}^p \xi_{t,i}^T \xi_{t,i} \label{eq:asyncrecursion}
\end{align}
where in the last line, we have dropped terms that are linear in $\xi_{t,i}$ by using the zero mean assumption. Making use of the strong convexity inequality \eqref{eq:stronglyconvexssumption}, we can bound the second term in the above sum by 
\[ -\nabla f(\theta_{t,i})^T (\theta_{t,i} - \theta_\ast) \leq -\frac{mL}{m+L} (\theta_{t,i} - \theta_\ast)^T (\theta_{t,i} - \theta_\ast) - \frac{1}{m+L} \nabla f(\theta_{t,i})^T \nabla f(\theta_{t,i}) \]
and rearrange the terms in \eqref{eq:asyncrecursion} to derive
\begin{align}
\mathbb{E}[\Vert \theta_{t+1} - \theta_\ast \mathbf{1} \Vert^2 \vert \theta_t] &\leq \left(1 - \frac{2 \alpha_t}{p} \frac{mL}{m+L} \right) \Vert \theta_t - \theta_\ast \mathbf{1} \Vert^2 \\
&+ \left( \frac{\alpha_t^2}{p} - \frac{2 \alpha_t}{p} \frac{1}{m+L} \right) \sum_{i=1}^p \nabla f(\theta_{t,i})^T \nabla f(\theta_{t,i}) \label{eq:asyncdisappear} \\
&+ \alpha_t^2 \sigma^2 
\end{align}
The term in \eqref{eq:asyncdisappear} can be dropped if $\alpha_t$ is chosen sufficiently small, with the now familiar requirement that $0 < \alpha_t < \frac{2}{m+L}$. Assuming this is the case, we have
\begin{align}
\mathbb{E}[\Vert \theta_{t+1} - \theta_\ast \mathbf{1} \Vert^2 \vert \theta_t] &\leq \left(1 - \frac{2 \alpha_t}{p} \frac{mL}{m+L} \right) \Vert \theta_t - \theta_\ast \mathbf{1} \Vert^2 + \alpha_t^2 \sigma^2  \label{eq:asyncrecurse} 
\end{align}

Assuming a constant step size $\alpha_t \equiv \alpha$, we can use the law of iterated expectations to unroll the recursion, giving
\begin{align}
\mathbb{E}[\Vert \theta_t - \theta_\ast \mathbf{1} \Vert^2] \leq \left( 1 - \frac{2\alpha}{p} \frac{mL}{m+L} \right)^t \Vert \theta_0 - \theta_\ast \mathbf{1} \Vert^2 + p \alpha \sigma^2 \frac{m+L}{2mL}
\end{align}
Note this is the same convergence rate guarantee as stochastic gradient descent with an extra factor of $p$, albeit the coordination amongst processors required to adjust the step size is unrealistic in practice. \\

Next, we prove the bound on the processors' local parameters' convergence to each other / to the mean, $\mathbb{E}[\Vert \theta_{t} - \bar{\theta}_t \mathbf{1} \Vert^2]$. \\

First, we write the spatial parameter average in vector form as 
\begin{align}
\bar{\theta}_t = \frac{1}{p} \mathbf{1}^T \theta_t 
\end{align}

With $D$ as previously defined in \eqref{eq:defofD}, we can write
\begin{align}
\bar{\theta}_{t+1} \mathbf{1} &= \frac{1}{p} \mathbf{1}^T \theta_{t+1} \mathbf{1} = \frac{1}{p} \mathbf{1}^T D(\theta_t - \alpha g_t) \mathbf{1} 
\end{align}
and so the term we wish to bound can be expanded as
\begin{align}
\Vert \theta_{t+1} - \bar{\theta}_{t+1} \mathbf{1} \Vert^2 &= (\theta_t - \alpha_t g_t)^T \left( D^T D - \frac{1}{p} D^T \mathbf{1} \mathbf{1}^T D \right) (\theta_t - \alpha_t g_t) \label{eq:asyncboundterm}
\end{align}
In addition to the diagonalization we previous calculated in \eqref{eq:asyncdiag} (reproduced below for convenience), we calculate some other useful diagonalizations. 
\begin{align}
\mathbb{E}[D^TD] &= \left( 1 - \frac{2 \beta (1-\beta)}{p} \right) \mathbf{I} + \frac{2 \beta (1-\beta)}{p^2} \mathbf{1} \mathbf{1}^T \\
&= Q \begin{bmatrix} 1 & & & \\ & 1-\frac{2\beta(1-\beta)}{p} & & \\ & & \ddots & \\ & & & 1-\frac{2\beta(1-\beta)}{p} \end{bmatrix} Q^T \\
\mathbb{E}[D^T \mathbf{1} \mathbf{1}^T D] &= \frac{2 \beta^2}{p} \mathbf{I} + \left( 1 - \frac{2 \beta^2}{p^2} \right) \mathbf{1} \mathbf{1}^T \\
&= Q \begin{bmatrix} p & & & \\ & \frac{2\beta^2}{p} & & \\ & & \ddots & \\ & & & \frac{2 \beta^2}{p} \end{bmatrix} Q^T \\ 
\mathbb{E}\left[ D^T D - \frac{1}{p} D^T \mathbf{1} \mathbf{1}^T D \right] &= \left( 1 - \frac{2 \beta (1-\beta)}{p} - \frac{2 \beta^2}{p^2} \right) \mathbf{I} + \left( \frac{2 \beta (1-\beta)}{p^2} - \frac{1}{p} \left( 1 - \frac{2 \beta^2}{p^2} \right) \right) \mathbf{1} \mathbf{1}^T \\
&= Q \begin{bmatrix} 0 & & & \\ & 1 - \frac{2 \beta(1-\beta)}{p} - \frac{2\beta^2}{p^2} & & \\ & & \ddots & \\ & & & 1 - \frac{2 \beta(1-\beta)}{p} - \frac{2\beta^2}{p^2} \end{bmatrix} Q^T \label{eq:asyncdiagtwo}
\end{align}
We note that all three matrices are diagonalized by the same orthogonal matrix $Q$. Furthermore, the first eigenvector, corresponding to the eigenvalues $1$, $p$, $0$ respectively, is $q_1 = \frac{1}{\sqrt{p}}\mathbf{1}$. \\

Continuing from \eqref{eq:asyncboundterm}, we take the expectation and use \eqref{eq:asyncdiagtwo} to further bound the expression. In the computation below we use $\lambda \equiv 1 - \frac{2\beta(1-\beta)}{p} - \frac{2 \beta^2}{p}$ as a notational shorthand.
\begin{align}
\nonumber & \mathbb{E}[\Vert \theta_{t+1} - \bar{\theta}_{t+1} \mathbf{1} \Vert^2 \vert \theta_{t}] \\
&=  \mathbb{E} \left[ (\theta_t - \alpha_t g_t)^T \left( D^T D - \frac{1}{p} D^T \mathbf{1} \mathbf{1}^T D \right) (\theta_t - \alpha_t g_t) \right] \\
&\leq \mathbb{E}\left[ \lambda \Vert (\theta_t - \bar{\theta}_t \mathbf{1}) - \alpha_t g_t \Vert^2 \right] \\
&= \lambda \left( \Vert \theta_t - \bar{\theta}_t \mathbf{1} \Vert^2 - \frac{2 \alpha_t}{p} \sum_{i=1}^p \nabla f(\theta_{t,i})^T (\theta_{t,i} - \bar{\theta}_t) + \frac{\alpha_t^2}{p} \sum_{i=1}^p \nabla f(\theta_{t,i})^T \nabla f(\theta_{t,i}) + \frac{\alpha_t^2}{p} \sum_{i=1}^p \xi_{t,i}^T \xi_{t,i} \right) \label{eq:strongconvexity} \\
&\leq \lambda \left( \Vert \theta_t - \bar{\theta}_t \mathbf{1} \Vert^2 + \frac{2 \alpha_t}{p} \sum_{i=1}^p \left( f(\bar{\theta}_t) - f(\theta_{t,i}) \right) - \frac{\alpha_t m}{p} \sum_{i=1}^p (\theta_{i,t} - \bar{\theta}_t)^2 + \frac{\alpha_t^2}{p} \sum_{i=1}^p (C^2+\sigma^2) \right)  \label{eq:convexity} \\
&\leq \lambda \left( 1- \frac{\alpha_t \, m}{p} \right) \Vert \theta_t - \bar{\theta}_t \mathbf{1} \Vert^2 + \lambda \alpha_t^2 (C^2+\sigma^2) 
\end{align}
In \eqref{eq:strongconvexity} and \eqref{eq:convexity}, we have used the definition of strong convexity and convexity respectively. \\

Finally, taking constant step size $\alpha_t \equiv \alpha$, and using the law of iterated expectations to unroll the recursion, we have
\begin{align}
\mathbb{E}[\Vert \theta_t - \bar{\theta}_t \mathbf{1} \Vert^2] &\leq \left( \lambda \left( 1- \frac{\alpha \, m}{p} \right) \right)^t \Vert \theta_0 - \bar{\theta}_0 \mathbf{1} \Vert^2  + \frac{\lambda \alpha^2 (C^2+\sigma^2)}{1 - \lambda \left( 1- \frac{\alpha \, m}{p} \right)}
\end{align}
When the step size $\alpha$ is small and the number of processors $p$ is large, the quantity $\lambda \left( 1- \frac{\alpha \, m}{p} \right) = \left( 1 - \frac{2\beta(1-\beta)}{p} - \frac{2 \beta^2}{p} \right) \left( 1- \frac{\alpha \, m}{p} \right)$ is well approximated by $1 - \frac{2 \beta (1-\beta)}{p} - \frac{2 \beta^2}{p} - \frac{\alpha m}{p}$, which makes clear the dependence of the rate on the parameter $p$.
\end{proof}

\subsection{Gossip variants}


\subsubsection{Gossip with stale parameters (gradient step and gossip at the same time)}

Consider a one-step distributed consensus gradient update:
\begin{align}
  \theta_{i,t+1}
  &=\theta_{i,t} - \alpha\nabla f(\theta_{i,t};X_i) - \beta\left( \theta_{i,t}-\frac{1}{p}\sum_{j=1}^p \theta_{j,t} \right).
\end{align}
If we replace the distributed mean $\frac{1}{p}\sum_{j=1}^p\theta_{j,t}$ with
an unbiased one-sample estimator $\theta_{j_{i,t},t}$,
such that $j_{i,t}\sim\text{Uniform}(\{1,\ldots,p\})$
and $\mathbb E[\theta_{j_{i,t},t}]=\frac{1}{p}\sum_{j=1}^p\theta_{j,t}$,
then we derive the gossiping SGD update:
\begin{align}
  \theta_{i,t+1}
  &=\theta_{i,t} - \alpha\nabla f(\theta_{i,t};X_i) - \beta ( \theta_{i,t}-\theta_{j_{i,t},t} ) \\
  &=(1-\beta)\theta_{i,t} + \beta\theta_{j_{i,t},t} - \alpha\nabla f(\theta_{i,t};X_i).
\end{align}

\subsubsection{Gossip with fresh parameters (gradient step before gossip)}
\label{sec:fresh}


Consider a two-step distributed consensus gradient update:
\begin{align}
  \theta_{i,t}'
  &=\theta_{i,t} - \alpha\nabla f(\theta_{i,t};X_i) \\
  \theta_{i,t+1}
  &=\theta_{i,t}' - \beta\left( \theta_{i,t}'-\frac{1}{p}\sum_{j=1}^p \theta_{j,t}' \right).
\end{align}
If we replace the distributed mean $\frac{1}{p}\sum_{j=1}^p\theta_{j,t}'$ with
an unbiased one-sample estimator $\theta_{j_{i,t},t}'$,
such that $j_{i,t}\sim\text{Uniform}(\{1,\ldots,p\})$
and $\mathbb E[\theta_{j_{i,t},t}']=\frac{1}{p}\sum_{j=1}^p\theta_{j,t}'$,
then we derive the gossiping SGD update:
\begin{align}
  \theta_{i,t}'
  &=\theta_{i,t} - \alpha\nabla f(\theta_{i,t};X_i) \\
  \theta_{i,t+1}
  &=\theta_{i,t}' - \beta ( \theta_{i,t}'-\theta_{j_{i,t},t}' ) \\
  &=(1-\beta)\theta_{i,t}' + \beta\theta_{j_{i,t},t}'.
\end{align}

\subsection{Implementation details}

We provide some more details on our implementation of deep convolutional
neural network training in general.

\subsubsection{ImageNet data augmentation}

We found that multi-scale training could be a significant performance bottleneck
due to the computational overhead of resizing images,
even when using multiple threads and asynchronous data loading.
To remedy this, we used fast CUDA implementations of linear and cubic
interpolation filters to perform image scaling during training on the GPU.
We also preprocessed ImageNet images such that their largest dimension was
no larger than the maximum scale (in our case, 480 pixels).

\subsubsection{ResNet implementation}

We implemented ResNet-18 using stacked residual convolutional layers with
$1\times1$ projection shortcuts.
We used the convolution and batch normalization kernels from cuDNNv4.
The highest ImageNet validation set accuracy (center crop, top-1) our
implementation of ResNets achieved was about 68.7\% with the aforementioned
multi-scale data augmentation;
we note that researchers at Facebook independently reproduced ResNet with a
more sophisticated data augmentation scheme and achieved 69.6\% accuracy
using the same evaluation methodology on their version of ResNet-18.%
\footnote{
  \texttt{https://github.com/facebook/fb.resnet.torch}
}

\end{document}